\documentclass[10pt]{article} % For LaTeX2e

\usepackage[preprint]{rlj}           % Should be uncommented for submission
\usepackage{amssymb}            % Defines common symbols like \mathbb R
\usepackage{mathtools}          % Extends amsmath, providing common math tools
\usepackage{mathrsfs}           % Enables \mathscr, which can work in cases that \mathcal does not
\usepackage{graphicx}           % For including images
\usepackage{subcaption}         % Allows for the use of subfigures and subcaptions
\usepackage[space]{grffile}     % For spaces in image names
\usepackage{url}                % For displaying URLs
\usepackage{lipsum}             % For placeholder text

\usepackage{algorithm}
\usepackage{algorithmic}
\usepackage{amsthm,amsmath,amsfonts,bm}

\def\eqref#1{equation~\ref{#1}}
\def\floor#1{\lfloor #1 \rfloor}
\def\1{\bm{1}}

\DeclareMathAlphabet{\mathsfit}{\encodingdefault}{\sfdefault}{m}{sl}
\SetMathAlphabet{\mathsfit}{bold}{\encodingdefault}{\sfdefault}{bx}{n}

\def\gA{{\mathcal{A}}}

\def\gH{{\mathcal{H}}}

\def\gS{{\mathcal{S}}}

\def\sR{{\mathbb{R}}}

\DeclareMathOperator*{\argmax}{arg\,max}

\newtheorem{theorem}{Theorem}
\newtheorem{lemma}{Lemma}

\title{On the Convergence of Monte Carlo UCB for Random Length Episodic MDPs}

\setrunningtitle{MC-UCB}

\author{Zixuan Dong\textsuperscript{1,2}, Che Wang\textsuperscript{2,3,$\dagger$}, Keith Ross\textsuperscript{4}}

\emails{\{zixuandong, chewang, keithwross\}@nyu.edu}

\affiliations{
$^{1}$\textbf{SFSC of AI and DL, NYU Shanghai}\\
$^{2}$\textbf{New York University}\\
$^{3}$\textbf{New York University Shanghai}\\
$^{4}$\textbf{New York University Abu Dhabi}
\par % If including additional comments like below, use \par to add some whitespace. 
$^\dagger$ This work was done prior to Che Wang joining Amazon.
}

\contribution{
    We prove that the Q and V estimates produced by MC-UCB converge almost surely to the optimal Q and V functions, for any Optimal Policy Feed-Forward (OPFF) MDPs, which covers all deterministic MDPs, as well as canonical stochastic environments such as Blackjack.
    }
    {
    For random-length episodes with exploring starts, MC convergence was generally unproven or known to fail unless strict assumptions were made \citep{bertsekas1996neuro, wang2022on}.  In contrast to exploring starts, incorporating an Upper Confidence Bound (UCB) exploration mechanism provides a more practical condition for MC methods to converge almost surely in OPFF MDPs, in which no states are revisited by the optimal policy—even though suboptimal policies may revisit states.
    }

\contribution{
    The almost sure convergence result for random-length OPFF MDPs directly extends to all finite-horizon MDPs by encoding the time step into the state.
    }
    {
    Classic regret bounds for MC-UCB were studied primarily under finite horizon. Our analysis subsumes that case under a broader OPFF framework, proving almost sure convergence rather than bounding expected performance. 
    }

\contribution{
    We empirically compare MC-UCB and MCES on two classic OPFF tasks—Blackjack and Cliff-Walking—and find that MC-UCB achieves competitive performance and convergence despite avoiding the impractical assumption of exploring starts.
    }
    {
    While MCES relies on exploring starts for its strong performance, this assumption is often infeasible. By substituting UCB exploration, MC-UCB remains practically viable and still attains results on par with MCES.
    }

 % Your keywords

\summary{Monte Carlo (MC) reinforcement learning algorithms learn by averaging episodic returns to update the Q-function. One well-known variant, Monte Carlo Exploring Starts (MCES), relies on repeatedly initializing each state-action pair, which is often infeasible in practice. In contrast, Monte Carlo UCB (MC-UCB) removes this requirement by adding an Upper Confidence Bound (UCB) exploration term to its policy. Although prior work on MC-UCB variants has focused on the finite-horizon regret analysis, MCES is known to fail to converge in some MDPs, and we similarly conjecture that MC-UCB may not converge in general. We nevertheless prove that the Q function estimated by MC-UCB does almost surely converge to the optimal Q function for a broad class of random-length episodic MDPs, including canonical deterministic domains such as Go and stochastic environments such as Blackjack.
}

\begin{document}

% \makeCover  % Create the cover page
\maketitle  % Make the title section

\begin{abstract}
In reinforcement learning, Monte Carlo algorithms update the Q function by averaging the episodic returns. In the Monte Carlo UCB (MC-UCB) algorithm, the action taken in each state is the action that maximizes the Q function plus an Upper Confidence Bounds (UCB) exploration term, which biases the choice of actions to those that have been chosen less frequently. Although there has been significant work on establishing regret bounds for MC-UCB, most of that work has been focused on finite-horizon versions of the problem, for which each episode terminates after a constant number of steps. For such finite-horizon problems, the optimal policy depends both on the current state and the time within the episode. However, for many natural episodic problems, such as games like Go and Chess and robotic tasks, the episode is of random length and the optimal policy is stationary. For such environments, it is an open question whether the Q-function in MC-UCB will converge to the optimal Q function; we conjecture that, unlike Q-learning, it does not converge for all MDPs. We nevertheless show that for a large class of MDPs, which includes stochastic MDPs such as blackjack and deterministic MDPs such as Go, the Q function in MC-UCB converges almost surely to the optimal Q function. An immediate corollary of this result is that it also converges almost surely for all finite-horizon MDPs. We also provide numerical experiments, providing further insights into MC-UCB. 
\end{abstract}

\section{Introduction}
One of the most fundamental results in tabular reinforcement learning (RL) is that the Q-function estimator for Q-learning converges almost surely to the optimal Q-function, without any restriction on the environment. This classic result was originally proved for infinite-horizon discounted MDPs, for which the the optimal policy is both stationary and deterministic \citep{tsitsiklis1994asynchronous, jaakkola1994convergence}. Q-learning is based on dynamic programming and uses back-ups to update the estimates of the optimal action-value function.

Besides Q-learning, another major class of RL algorithms is Monte Carlo (MC) algorithms \citep{Sutton1998}, which are often employed for episodic MDPs. With MC algorithms, an optimal policy is sought by repeatedly running episodes and recording the episodic returns. After each episode, the estimate for the optimal Q function is updated using the episodic return. Then a new episode is generated using the most recent estimate of the Q-function. In contrast with Q-learning, pure MC algorithms do not employ backups. 

For many natural episodic problems, such as games like Go and Chess and robotic tasks, the episode ends when a terminal state is reached. The episode therefore is of random length. For episodic environments with random length episodes, the natural optimization criterion is to maximize the (expected) total return until the episode ends. For this class of problems, it is well known in the MDP literature that there is an optimal policy that is both stationary (does not depend on the time index) and deterministic \citep{bertsekas1991analysis}. 

As with Q-learning, it is important to understand the convergence properties of Monte Carlo algorithms with random-length episodes. We note that MC algorithms are often used in practice. For example, training AlphaZero \citep{silver2018general} employs Monte Carlo in both the inner loop (where Monte Carlo Tree Search (MCTS) is used) and in the outer loop, where a two-player MC episode is run until game termination, at which time estimates for the Q functions are updated. As it turns out, unlike Q-learning, MC algorithms are not guaranteed to converge. Indeed, counterexamples exist for non-convergence for an important class of MC algorithms \citep{bertsekas1996neuro, wang2022on}. 

In this paper, we consider an important class of MC algorithms, namely, MC algorithms that use Upper-Confidence Bounds (UCB) exploration at each state. In practice, count-based UCB exploration has been used for many important tasks, including tree search and online RL with human feedback \citep{silver2018general, pmlr-v125-jin20a, bellemare2016unifying, copo}. As discussed in the related work, such algorithms have been studied extensively for the fixed horizon optimization criterion, for which the optimal policy is non-stationary. Here we explore MC-UCB algorithms for random-length episodes, which is often the more natural setup in practice. For such environments, although the optimal policy is stationary, we conjecture that MC-UCB does not always converge. 

The main contribution of this paper is to establish almost sure convergence for MC-UCB for a very large and important class of environments, specifically, for environments whose states are never revisited when employing an optimal policy. This class of MDPs includes stochastic MDPs such as blackjack and many gridworld environments with randomness. It also includes all deterministic environments and all episodic environments with a monotonically changing value in the state. The proof is based on induction and a probabilistic argument invoking the strong law of large numbers. 

A second contribution of this paper is the numerical study of MC-UCB. These results show that in two classic episodic environments, MC-UCB can achieve competitive performance and converge in both the policy and the value function, when compared to an ideal but less practical exploration scheme where each episode must be initialized with a random state-action pair. 

\section{Preliminaries}
RL typically aims to find the optimal policy for a  Markov Decision Process (MDP). A finite MDP can be represented by a tuple $\{\gS, \gA, r, P\}$, where $\gS$ is the finite state space; $\gA$ is the finite action space; $r$ is the reward function from $\gS\times\gA$ to the reals; and $P$ is the dynamics function where for any $s\in\gS$ and $a\in\gA$, $P(\cdot|s, a)$ gives a probability distribution over the state space $\gS$. We denote $\pi$ for the policy. In general a policy may depend on the current state, the current time and the entire past history. Denote $s_t^{\pi}$ and $a_t^{\pi}$ for the state and action at time $t$ under policy $\pi$, respectively. 

We assume that the state space has a {\em terminal state} denoted by $\tilde{s}$. When the MDP enters state $\tilde{s}$, the MDP terminates. Let $T^{\pi}$ denote the time when the MDP enters $\tilde{s}$. $T^{\pi}$ is a random variable, and can be infinite if the MDP doesn't reach the terminal state. The expected total reward when starting in state $s$ is
\begin{equation}
    V_{\pi}(s) = E[\sum_{t=0}^{T^{\pi}} r(s_t^{\pi}, a_t^{\pi}) | s_0 =s  ]
    \label{criterion}
\end{equation}
Similarly we define the action-value function as
\begin{equation}
    Q_{\pi}(s,a) = E[\sum_{t=0}^{T^{\pi}} r(s_t^{\pi}, a_t^{\pi}) | s_0 =s, a_0 = a  ]
    \label{Qcriterion}
\end{equation}
Further denote
\[
Q^*(s,a) = \max_{\pi}Q^{\pi}(s,a)
\]
\[
V^*(s) = \max_{\pi}V^{\pi}(s) = \max_a Q^*(s,a)
\]

Let $\pi^*$ denote a policy that maximizes (\ref{criterion}). We say that the MDP is {\em episodic} if under any optimal policy $\pi^*$, $P(T^{\pi^*} < \infty| s_0 =s) =1$ for all $s$.
Note that in this definition of an episodic MDP, episodes are of random-length.
Maximizing (\ref{criterion}) for an episodic MDP is sometimes referred to as the stochastic shortest path problem \citep{bertsekas1991analysis}. For the stochastic shortest path problem, there exists an optimal policy that is both stationary
and deterministic, that is, a policy of the form $\pi^*: \gS \rightarrow \gA $.

Another popular optimization criterion is the finite-horizon criterion, where $T^{\pi}$ in (\ref{criterion}) is replaced with a fixed deterministic value $h$. In this case rewards accumulate to time $h$ whether the MDP terminates or not. For finite horizon problems, optimal policies are typically not stationary; however, it can be shown that there exists an optimal policy that depends on the current state $s$ and the current time $t$ \citep{bertsekas1991analysis}.  

An MDP is said to be {\bf Optimal Policy Feed-Forward (OPFF)} if under any optimal policy a state is never revisited \citep{wang2022on}. Note that this definition allows for MDPs for which states are revisited under non-optimal policies. Many important MDPs are OPFF, including the stochastic environment Blackjack, the environments examined at the end of this paper, and all deterministic environments.  

\subsection{Monte Carlo with Upper Confidence Bound Framework}
\label{sec:mcucb}
In Monte Carlo RL, an optimal policy is sought by repeatedly running episodes and recording the episodic returns. After each episode, the estimate for the optimal Q function, denoted here as $Q(s,a)$, is updated using the episodic return. Then a new episode is generated using the most recent estimate of the Q-function $Q(s,a)$.

There are many possible variations of Monte Carlo RL algorithms. Two broad classes are MC with Exploring Starts (MCES) and MC using Upper Confidence Bounds (MC-UCB). In MCES, the policy used to select actions when in state $s$ is simply $a = \argmax Q(s,a)$, that is, the action that maximizes the current estimates of the Q function. In MCES, exploration is performed by starting the episodes in randomly chosen $(s,a)$ pairs. One of the major drawbacks of MCES is that it assumes that during training we have the ability to choose the starting state and action for an episode. This is not a feasible option for many real-world environments. MC-UCB does not have this drawback; for example, during training, each episode may start in a given fixed state. In MC-UCB, exploration is performed by adding an exploration term to $Q(s,a)$, then taking the $\argmax$ of the sum to select an action. It is very similar to how UCB is done with multi-armed bandits, but in the MDP case, the Q estimate and the exploration term depend on the state. 

\begin{algorithm}[h]
\caption{MCUCB}
%for Stochastic Feed-Forward MDPs}
	\label{alg:mc-ucb}
	\begin{algorithmic}[1]
	\STATE Initialize: $\pi(s) \in \gA$ arbitrarily, $Q(s,a) = 0 \in \sR$, for all $s\in \gS, a \in \gA$; \\  $Returns(s,a) \leftarrow$ empty list, for all $s\in \gS, a \in \gA$.
	\WHILE{True} \label{line:while-loop-forever}
	%\STATE Choose $S_0 \in \gS$, $A_0 \in \gA$  s.t. all pairs are chosen infinitely often.\label{line:choose-s0}
	%\label{line:explore}
	
	\STATE Starting from the given initial state $s_0$, generate an episode following $\pi$: $s_0, a_0, s_1,a_1,\ldots,s_{T-1}, a_{T-1},s_{T}$.
	\STATE $G \leftarrow 0$
	\FOR{$t = T-1, T-2,\dots, 0$}
    \STATE $G \leftarrow G + r(s_t, a_t)$
    \STATE $N(s_t)\leftarrow N(s_t)+1$
    \STATE $N(s_t,a_t)\leftarrow N(s_t,a_t)+1$
    \STATE Append $G$ to $Returns(s_t, a_t)$
    \STATE $Q(s_t, a_t) \leftarrow average(Returns(s_t, a_t))$ \label{line:q-update}
    \STATE $\pi(s_t) \leftarrow argmax_a Q(s_t, a) + C\sqrt{\frac{\log N(s_t)}{N(s_t, a)}}$.
    \label{line:policy-update}
	\ENDFOR
	\ENDWHILE
	\end{algorithmic}
\end{algorithm}

In this paper we focus on the convergence properties of MC-UCB for MDPs with random-length episodes where rewards accrue until the episode enters the terminal state, as is typically the case in real applications. Algorithm \ref{alg:mc-ucb} provides a basic form of the MC-UCB algorithm.

In the algorithm, after every episode, the optimal action-value function $Q(s,a)$ for a particular state-action pair $(s,a)$ is  estimated by averaging all the sampled returns seen for that pair. When running an episode, when in state $s$, we select the action that maximizes the sum of $Q(s,a)$ and the Upper Confidence Bound exploration term $\sqrt{\frac{C\log N(s)}{N(s, a)}}$, where $C$ is a constant controlling the exploration rate, $N(s)$ is the total number of visits for a state $s$, and $N(s, a)$ is the total number of visits for a state-action pair $(s, a)$. We emphasize that the above framework is for environments where states are never re-visited. We will later modify the algorithm for OPFF environments where cycles are possible. 

Because the policy for generating episodes changes after each episode, we cannot directly apply the Strong Law of Large Numbers (SLLN) to prove almost sure convergence; indeed the returns collected for a given state-action pair are neither independently distributed nor identically distributed. 

%Therefore, before we dive into the analysis for the more general MDP setting, we first start with a Multi-arm Bandit (MAB) setting as a base case, where we show that averaged payoffs of pulling arms following UCB policy actually converges to the expected payoff of the optimal arm with probability one as trials go to infinity.

\section{Related Work}
In contrast with Monte Carlo methods, Q-learning is based on dynamic programming and uses back-ups to update the estimates of the optimal action-value function. For infinite-horizon discounted MDPs, Q-learning can be shown to converge with probability one to the optimal action-value function under general conditions, by putting Q-learning into the stochastic approximations framework \citep{tsitsiklis1994asynchronous, jaakkola1994convergence}. Moreover, for the stochastic shortest path problem (random-length episodes), it has been shown that the sequence of Q-learning iterates is bounded, implying the convergence of the action-value function \citep{yu2013boundedness}. While the MC-UCB algorithm also aims to solve the stochastic shortest path problem, the proofs of convergence for Q-learning cannot be easily adapted for Monte Carlo algorithms. 
The MCES algorithm is related to the MC-UCB algorithm since both MCES and MC-UCB average Monte Carlo episodic returns to estimate the action-value function. However, the two algorithms employ very different exploration strategies. While MCES randomly chooses state-action pairs at the beginning of each episode, MC-UCB explores by adding UCB exploration to the current action-value estimate. In the general RL setting, the convergence of MCES is not guaranteed, since there exists a counterexample where MCES indeed does not converge to the optimal action-value function \citep{bertsekas1996neuro, wang2022on}. We conjecture that the MC-UCB algorithm also does not converge in general. 

However, if the MCES algorithm is modified so that the Q-value estimates are updated at the same rate for all state-action pairs, and the discount factor is strictly less than one, then the convergence will be guaranteed \citep{tsitsiklis2002convergence}. Furthermore, by assuming that all policies are proper, it is shown via the stochastic approximations methodology that MCES also converges for undiscounted cases \citep{chen2018convergence, liu2020convergence}. More recently, a new but simple approach \citep{wang2022on} to this problem relaxes the assumption of a uniform update rate for all state-action pairs, providing almost sure convergence under the OPFF assumption, that is, a state is never re-visited within an episode under an optimal policy. 
In this paper we pursue this same strategy for the MC-UCB algorithm. 
However, we emphasize here that our proof is significantly more challenging than the proof in \citet{wang2022on} due to the nature MC-UCB algorithm.

There has been significant work on regret bounds for the finite-horizon MDP problems \citep{Chang05anadaptive, pmlr-v70-azar17a, agarwal2019reinforcement}. When the horizon is finite, then the optimal policy is non-stationary. Furthermore, many problems of practical interest do not have a fixed-length horizon. We also note regret bounds typically only imply convergence in expectation, not almost sure convergence. 
Although regret bounds are important, we argue that it is first important to understand the convergence properties of the MC-UCB algorithm since convergence is not necessarily guaranteed as it is in finite-horizon problems. 

%The MC-UCB algorithm can be viewed as a generalized framework of several UCB-related algorithms designed for the finite-horizon MDP problem. The Adaptive Sampling algorithm \cite{Chang05anadaptive}, for example, is a natural extension of the UCB1 algorithm from the stochastic bandit problem to the finite-horizon MDP problem, where the authors establish the convergence in expectation of the state-value estimated by their algorithm. The UCBVI algorithm \cite{pmlr-v70-azar17a} employs the UCB as a bonus when updating the Q values. However, they also take a model-based approach by using the trajectory information to estimate the transition probabilities of the MDP. 
%Moreover, their theoretical results about regret bounds are for finite-horizon MDPs and only imply convergence in expectation for value functions. 

We also note that the MC-UCB framework shares some similarities with the Monte Carlo Tree Search (MCTS) algorithms. According to \citet{10.5555/3207692.3207712}, the UCT algorithm \citep{10.1007/11871842_29}, one of the fundamental MCTS algorithms, can be viewed as an offline on-policy every-visit MC control algorithm that uses the UCB1 (a variant of UCB, discussed in appendix \ref{appendix:mab}) as the policy.  The MCTS algorithms can also be considered as finite-horizon since they truncate the complete game trajectories at a fixed length.

\section{Almost Sure Convergence for OPFF MDPs}
In this section, we return to the MC-UCB framework described in Section \ref{sec:mcucb} and derive the almost sure convergence of both the action-value function and state-value function under the OPFF assumption. 

Before we present the almost sure convergence proof for the OPFF setting, we note that OPFF environments allow the learning agent to revisit previous states for sub-optimal policies; therefore, the agent may get into an infinite loop where the terminal state is never reached. To address this issue, we set a time limit $M$ depending on the number of states in the MDP, as shown in Algorithm \ref{alg:mcucb-opff}. If the length of an episode exceeds this limit, we terminate the algorithm (line \ref{line:termination}). Additionally, there are some special corner cases where taking sub-optimal actions never yields a complete episode, such as actions leading to a self-loop. Therefore, we add lines \ref{line:detect-loop} and \ref{line:detect-loop-next} to the algorithm. Since the action-value function for optimal actions will be almost surely greater than that for sub-optimal actions, as we will show in the proof, we can generate a complete episode in those corner cases. Finally, we use the first-visit mechanism (line \ref{line:first-visit}) to calculate the average returns, since the every-visit mechanism may cause a slower convergence for tasks with cycles \citep{10.1007/BF00114726, Sutton1998}.

% Note that one can modify MC-UCB to make all Q estimates converge even faster. For example, similar to the model-based method \citep{pmlr-v70-azar17a}, one can estimate transition probabilities and V estimates, and then update Q estimates by the Bellman equation. In this case, our analytical methodology still applies. However, in order to follow the standard MC-UCB framework and to simplify the analysis in this paper, we use the approach mentioned above to deal with the theoretical convergence of all Q values. 

\begin{algorithm}[h]
\caption{First-visit MC-UCB for OPFF MDPs}
	\label{alg:mcucb-opff}
	\begin{algorithmic}[1]
	\STATE Initialize: $\pi(s) \in \gA$ arbitrarily, $Q(s,a) = 0 \in \sR$, for all $s\in \gS, a \in \gA$; \\  $Returns(s,a) \leftarrow$ empty list, for all $s\in \gS, a \in \gA$. 
	\WHILE{True} \label{line:opff-while-loop-forever}
    \WHILE{$t<M$ and the episode has not ended}\label{line:termination}
	\STATE Starting from the given initial state $s_0$, generate an episode following $\pi$: $s_0, a_0, s_1,a_1,\ldots,s_t, a_t$.
	\IF{loop detected for $(s_t,a_t)$} \label{line:detect-loop}
	\STATE $a_t\leftarrow \argmax_a Q(s_t,a)$ \label{line:detect-loop-next}
	\ENDIF
	\ENDWHILE
	
	\STATE $G \leftarrow 0$
	\FOR{$t = T-1, T-2,\dots, 0$}
    \STATE $G \leftarrow G + r(s_t, a_t)$
    \IF{$(s_t,a_t)$ does not appear in $s_0, a_0, s_1,a_1,\ldots,s_{t-1}, a_{t-1}$} \label{line:first-visit}
    \STATE $N(s_t)\leftarrow N(s_t)+1$
    \STATE $N(s_t,a_t)\leftarrow N(s_t,a_t)+1$
    \STATE Append $G$ to $Returns(s_t, a_t)$
    \STATE $Q(s_t, a_t) \leftarrow average(Returns(s_t, a_t))$ \label{line:opff-q-update}
    \STATE $\pi(s_t) \leftarrow \argmax_a Q(s_t, a) + \sqrt{C\frac{\log N(s_t)}{N(s_t, a)}}$.
    \ENDIF
    \label{line:opff-policy-update}
	\ENDFOR
	\ENDWHILE
	\end{algorithmic}
\end{algorithm}

Let $V_n(s)=\frac{1}{n}\sum_{t=1}^n G_t^s$ be the estimated state-value function of state $s$ by averaging total $n$ returns $\{G^s_t\}_{t=1}^n$, where $G_t^s$ is the $t$-th return collected in state $s$; $Q_n(s,a)=\frac{1}{T^s_a(n)}\sum_{u=1}^{T^s_a(n)} G_{a,u}^s$ be the estimated action-value function, where $G_u^{s,a}$ is the $u$-th return collected in state $s$ by choosing action $a$; $T_a^s(n)$ be the number of selections of action $a$ in the state $s$ out of total $n$ visits of that state. Let $\{r_u(s,a)\}_{u=1}^{T^k_{a^*}(n)}$ be finite $i.i.d$ reward random variables following an unknown law 
%$\gP(s,a)$ 
with an unknown but finite mean $\bar{r}(s,a)$. We now present the main result:

\begin{theorem}[\textbf{Almost sure convergence of MC-UCB for OPFF MDPs}]
\label{theorem:mcucb-opff-converge}
Suppose the MDP is OPFF. Then $V_n(s)$ converges to $V^*(s)$ $w.p.1$, for all $s \in \gS$; $Q_n(s,a)$ converges to $Q^*(s,a)$ $w.p.1$, for all $s \in \gS$ and $a\in\gA$.
\end{theorem}
\begin{proof}
Since the MDP is OPFF, its graph induced by an optimal policy is a directed acyclic graph (DAG).
%which has at least one topological ordering. 
So we can reorder the $K$ states so that by taking the optimal action from state $s_k$, we only transition to a state in the set $\{s_{k+1}, \dots, s_K\}$, for any $k\in\{1,\dots,K\}$. Here, $s_K$ is a terminal state.

The proof is then done by backward induction. For $s=s_K$, the result is trivially true. Suppose the statement is true for all states $s\in\{s_{k+1}, \dots, s_K\}$. We now show it is true for the state $s = s_k$ as well. For simplicity, we use $k$ to denote the state $s_k$ for any $k\in\{1,\dots,K\}$.

Note that we have the following decomposition:
\begin{align*}
    V_n(k) 
    &= \frac{1}{n}\sum_{t=1}^n G_t^k = \sum_a \frac{T_a^k(n)}{n}\cdot \frac{1}{T_a^k(n)}\sum_{u=1}^{T_a^k(n)} G_u^{k,a} = \sum_a \frac{T_a^k(n)}{n}Q_n(k,a)
\end{align*}

Let $a^*$ be an optimal action for state $k$.
We first show that the estimate $Q_n(k, a^*)$ converges to the optimal Q value $Q^*(k, a^*)$ almost surely for the optimal action $a^*$. We have:
\begin{align*}
    Q_n(k,a^*) &= \frac{1}{T^k_{a^*}(n)}\sum_{u=1}^{T^k_{a^*}(n)} (r_u(k,a^*) + \sum_{k'>k}\mathbb{I}\{S_u = k'\}G_u^{k'})\\
    &= \frac{1}{T^k_{a^*}(n)}\sum_{u=1}^{T^k_{a^*}(n)} r_u(k,a^*) +
      \sum_{k'>k}\frac{1}{T^k_{a^*}(n)}\sum_{u=1}^{T^k_{a^*}(n)}\mathbb{I}\{S_u = k'\}G_u^{k'}\\
    &= \frac{1}{T^k_{a^*}(n)}\sum_{u=1}^{T^k_{a^*}(n)} r_u(k,a^*) +  
      \sum_{k'>k}\frac{T_{a^*}^k(n, k')}{T^k_{a^*}(n)}\cdot \frac{1}{T_{a^*}^k(n, k')}\sum_{u=1}^{T^k_{a^*}(n)}\mathbb{I}\{S_u=k'\}G_u^{k'}
\end{align*}
where $S_u$ is the next state from state $k$ in the $u$-th iteration that contains the $(k,a^*)$ pair; and $T_{a^*}^k(n, k') = \sum_{u=1}^{T^k_{a^*}(n)}\mathbb{I}\{S_u=k'\}$.

Due to the UCB exploration, $(k, a^*)$ will be visited infinitely often, which implies $T^k_{a^*}(n)\overset{n\rightarrow\infty}{\rightarrow}\infty$. Hence, by SLLN, we have 

$$\frac{1}{T^k_{a^*}(n)}\sum_{u=1}^{T^k_{a^*}(n)} r_u(k,a^*)\overset{w.p.1}{\longrightarrow} \bar{r}(k,a^*) \quad \text{and} \quad \frac{\sum_u\mathbb{I}\{S_u=k'\}}{T^k_{a^*}(n)} \overset{w.p.1}{\longrightarrow} P(k'|k,a^*)$$

Moreover, by the inductive hypothesis, the V values for any states in the set $\{k+1,\dots,K\}$ will converge almost surely.
%if enough returns starting in those states are collected. 
Therefore, given that $\frac{1}{T_{a^*}^k(n, k')}\sum_{u=1}^{T^k_{a^*}(n)}\mathbb{I}\{S_u=k'\}G_u^{k'}$ is the average return starting in the state $k'>k$, it eventually converges to $V^*(k')$ almost surely, since the UCB exploration will visit the transition $(k, a^*, k')$ infinitely often.

In this way, we have
\begin{equation}
\begin{split}
\label{optimal action convergence}
    Q_n(k,a^*) \overset{w.p.1}{\longrightarrow} \bar{r}(k,a^*) + \sum_{k'>k}P(k'|k,a^*)\cdot V^*(k') = Q^*(k,a^*)
\end{split} 
\end{equation}

Since taking non-optimal action may lead to states visited before, where the inductive hypothesis does not apply, we cannot follow the same procedure to prove the almost sure convergence of $Q_n(k,a_i)$ for any $a_i\ne a^*$. Instead, we can first prove $V_n(k)$ converges almost surely by showing $\frac{T_{a_i}^k(n)}{n} \rightarrow 0$ for any $a_i\ne a^*$.

Let $\Pi$ denote the set of all deterministic policies. For any $a_i\ne a^*$, we can rewrite:
\begin{align*}
    Q_n(k,a_i) = \sum_{\pi\in\Pi}\frac{T^k_{a_i}(n;\pi)}{T^k_{a_i}(n)}\cdot\frac{1}{T^k_{a_i}(n;\pi)}\sum_{u=1}^{T^k_{a_i}(n;\pi)} G_u^{k,a_i}(\pi)
\end{align*}
where, when following a policy $\pi$, $T_a^k(n;\pi)$ denotes the number of visits of the pair $(k,a)$, out of total $n$ visits of the state $k$; $G_u^{k,a_i}(\pi)$ denotes the $u$-th return collected in state $k$ by choosing action $a_i$.
By SLLN , for any $\pi\in\Pi$ such that $T^k_{a_i}(n;\pi)\rightarrow \infty$, we have:
\begin{align*}
    \frac{1}{T^k_{a_i}(n;\pi)}\sum_{u=1}^{T^k_{a_i}(n;\pi)} G_{a_i,u}^k(\pi) \overset{w.p.1}{\longrightarrow} Q^{\pi}(k,a_i)\leq Q^*(k,a_i)
\end{align*}
where $Q^{\pi}(*,*)$ is the action-value function for policy $\pi$. The above result further implies that $w.p.1$:
\begin{align}
\label{optimal policy}
    \limsup_{n\rightarrow \infty}Q_n(k,a_i) \leq Q^*(k,a_i)
\end{align}
Moreover, we observe that for any $a_i\neq a^*$, there exists some $\epsilon_1 > 0$ such that:
\begin{align*}
   Q^*(k,a_i) + \epsilon_1 \leq Q^*(k,a^*)
\end{align*}

Let $\Omega$ be the underlying sample space. Then, the set 
$$\Lambda_{a_i} = \{\nu\in\Omega: \limsup_{n\rightarrow \infty}Q_n(k,a_i)(\nu) \leq Q^*(k,a_i)(\nu)\}$$ 

has probability measure 1 for any suboptimal actions $a_i\ne a^*$ by the inequality (\ref{optimal policy}); and the set 

$$\Lambda_{a^*} = \{\nu\in\Omega: Q_n(k,a^*)(\nu) \overset{n\rightarrow\infty}{\longrightarrow} Q^*(k,a^*)(\nu)\}$$ 

has probability measure 1 for the optimal action $a^*$ by the convergence result (\ref{optimal action convergence}). Letting $\Lambda = \cap_{a\in\gA}\Lambda_a$,  we have $P(\Lambda)=1$.

Now, we fix $\omega\in\Lambda$. By the definition of limit superior, for any $0<\epsilon_2 <\epsilon_1$, there exists $N_1$ such that for $n\geq N_1$, we have:
\begin{align*}
   Q_n(k,a_i)(\omega) \leq Q^*(k,a_i)(\omega) + \epsilon_2
\end{align*}
Hence, for any $n\geq N_1$, we have:
\begin{align*}
   Q^*(k,a^*)(\omega) &\geq Q^*(k,a_i)(\omega) + \epsilon_1 \\
   &\geq Q_n(k,a_i)(\omega) + \epsilon_1 - \epsilon_2
\end{align*}

Now, we are ready to show $\frac{T_{a_i}^k(n)}{n} \rightarrow 0$ $w.p.1$ for any $a_i\ne a^*$.
Following a standard upper bound in Multi-armed Bandit literature (for example, see inequality 8.4 of section 8.1 of \citet{lattimore_szepesvari_2020}), we have for any $\epsilon_3<\epsilon_1 - \epsilon_2$:
\begin{align*}
    T_{a_i}^k(n) & \leq \sum_{l=1}^n \mathbb{I}\{Q_s^{k,a_i} + \sqrt{\frac{C\log n}{l}} \geq Q^*(k,a^*) - \epsilon_3\}\\
    & + \sum_{t=1}^n \mathbb{I}\{Q_{t-1}(k,a^*) +\sqrt{\frac{C\log t}{T_{a^*}^k(t-1)}} \leq Q^*(k,a^*) - \epsilon_3\}
\end{align*}
where $Q_l^{k,a_i}$ is defined to be the estimated action-value function of the pair $(k,a_i)$ by using the first $l$ samples stored in the $Returns(k,a_i)$ array; if $l>T_{a_i}^k(n)$, we will just set $Q_l^{k,a_i} = Q_n(k,a_i)$.

For the first term on the right side of the above inequality, we have for $l\geq N_1$:
\begin{align*}
    &\sum_{l=1}^n \mathbb{I}\{Q_l^{k,a_i}(\omega) + \sqrt{\frac{C\log n}{s}} \geq Q^*(k,a^*)(\omega) - \epsilon_3\}\\
    &= \sum_{l=1}^n \mathbb{I}\{\sqrt{\frac{C\log n}{s}} \geq Q^*(k,a^*)(\omega)- Q_l^{k,a_i}(\omega)  - \epsilon_3\}\\
    &\leq N_1 + \sum_{l=N_1+1}^n \mathbb{I}\{\sqrt{\frac{2\log n}{s}} \geq \epsilon_1 - \epsilon_2  - \epsilon_3\}\\
    & = N_1 +\sum_{l=N_1+1}^n \mathbb{I}\{s \leq \frac{2\log n}{(\epsilon_1 - \epsilon_2  - \epsilon_3)^2}\}\\
    &= \floor{\frac{2\log n}{(\epsilon_1 - \epsilon_2 - \epsilon_3)^2}}
\end{align*}

On the other hand, choose $\epsilon_4 > 0$ such that $\epsilon_3 \geq \epsilon_4$. Then, there exists $N_2 > 0$ such that for any $t > N_2$, $|Q_t(k,a^*)(\omega)-Q^*(k,a^*)(\omega)| < \epsilon_4$, we can also bound the second term:
\begin{align*}
    &\sum_{t=1}^n \mathbb{I}\{Q_{t-1}(k,a^*)(\omega) + 
    \sqrt{\frac{C\log t}{T_{a^*}^k(t-1)(\omega)}} \leq Q^*(k,a^*)(\omega) - \epsilon_3\}\\
    &\leq \sum_{t=1}^n \mathbb{I}\{Q_{t-1}(k,a^*)(\omega) \leq Q^*(k,a^*)(\omega) - \epsilon_3\}\\
    &\leq \sum_{t=1}^n \mathbb{I}\{ \epsilon_3\leq |Q^*(k,a^*)(\omega) - Q_{t-1}(k,a^*)(\omega)|\}\\
    &\leq N_2 + \sum_{N_2+1}^n \mathbb{I}\{ \epsilon_3\leq \epsilon_4\}\\
    &= N_2
\end{align*}

From the above two inequalities, it is easy to see: for any $\omega\in\Lambda$, $\frac{T_{a_i}^k(n)(\omega)}{n} \rightarrow 0$ for any $a_i\ne a^*$. Since $P(\Lambda) = 1$, we have $\frac{T_{a_i}^k(n)}{n} \rightarrow 0$ $w.p.1$ for any $a_i\ne a^*$.

Together with the assumption that the Q values are finite, the almost sure convergence of the V values for the state $k$ follows:
\begin{align}
\label{optimal state value}
V_n(k) \overset{w.p.1}{\longrightarrow} \sum_{a_i\neq a^*}0\cdot Q_n(k,a_i) + 1 \cdot Q^*(k, a^*) = V^*(k)
\end{align}
Therefore, we can conclude that $V_n(s) \rightarrow V^*(s)$ $w.p.1$ for any $s\in\gS$.

Lastly, it remains to establish the almost sure convergence of the Q values for any sub-optimal action $a_i\ne a^*$. We have:
\begin{align*}
    &Q_n(k,a_i) = \frac{1}{T^k_{a_i}(n)}\sum_{u=1}^{T^k_{a_i}(n)} r_u(k,a_i) + \sum_{k'\in \gS}\frac{T_{a_i}^k(n, k')}{T^k_{a_i}(n)}\cdot \frac{1}{T_{a_i}^k(n, k')}\sum_{u=1}^{T^k_{a_i}(n)}\mathbb{I}\{S_u=k'\}G_u^{k'}
\end{align*}
If there is no potential for a loop after the transition $(k, a_i, k')$, the analysis for the almost sure convergence is the same as that for the Q value of the optimal action.
%, given that result that for any state, its V value will converge almost surely.
However, for those transitions leading to an infinite loop (e.g., a self-loop), the episode cannot reach the terminal state. Note that the convergence result (\ref{optimal action convergence}) and the inequality (\ref{optimal policy}) implies that $\lim_{n\rightarrow\infty}\argmax_a Q_n(s,a) = a^*$ $w.p.1$. Thanks to the line \ref{line:detect-loop} and \ref{line:detect-loop-next} in Algorithm \ref{alg:mcucb-opff}, if a loop is detected in the state $k$, then the remaining episode will almost surely be the same episode generated by $(k,a^*)$. In this case, a complete episode starting in the state $k'$ after choosing the suboptimal action $a_i$ can occur. It follows $Q_n(s,a)\rightarrow Q^*(s,a)$ $w.p.1$ for any $s\in\gS$ and for all $a\in\gA$.
\end{proof}

%Since in this paper, we focus on the stochastic shortest path setting as mentioned in the preliminary, the proposition $2.2$ of \cite{inbook} indicates that there exists  an optimal policy that is both stationary and deterministic. 
Note that the convergence result for the estimated Q-value to the optimal action-value function (\ref{optimal action convergence}) together with the inequality (\ref{optimal policy}) imply that the  induced policy $\pi_Q(s) = \argmax_a Q_n(s,a)$ converges almost surely to an optimal policy (which is stationary and deterministic).

We also emphasize that the theorem implies almost sure convergence for {\em all} MDPs for the (fixed) finite-horizon criterion (in which case the optimal policy is non-stationary). Given a finite-horizon period $\gH$ = \{0, 1, \dots, h\}, the policy $\pi(s,t)$, Q values $Q(s,a,t)$, and V values $V(s,t)$ are now time-dependent for $t\in\gH$. In this case, one can consider the time $t$ as part of the state, making the state space $\gS\times\gH$. Since the time step always increases within an episode, the algorithm never revisits previous states. Therefore, all finite-horizon MDPs are OPFF, and Theorem \ref{theorem:mcucb-opff-converge} applies.

\section{Numerical Experiments}
We now present additional numerical experiments comparing MC-UCB and MCES. 
% We run the MCUCB and the MCES algorithms on the two environments. 
We emphasize that this numerical study aims to bring a better understanding of our theoretical results. Note that, MCES uses a more ideal, but less practical exploration scheme: for each episode, the initial state-action pair has to be randomly sampled from the state-action space. While MC-UCB does not require this condition and can be applied to more practical tasks. 
% let us understand more about the theoretical observations we proposed, instead of achieving high algorithmic performance. 
To be more consistent with prior works, we follow \citet{wang2022on} and focus on two classic OPFF environments: Blackjack and Cliff-Walking. 

\subsection{Blackjack}
We first experiment on the classic card game Blackjack (discussed in detail in \citet{Sutton1998}). We compare the performance, policy correctness, and absolute update difference. The performance is simply the average return of the learned policy. The policy correctness is the total number of states where the learned policy matches the optimal policy computed by the value iteration algorithm \citep{Sutton1998}. The absolute update difference is the absolute differences between the action-value functions before and after an update, which can measure the convergence rate of the action-value function. The performance is evaluated over 5 episodes. All results are averaged over 5 seeds.

\begin{figure}[h]
    \centering
    \begin{subfigure}[b]{0.3\textwidth}
        \centering
        \includegraphics[width=\textwidth]{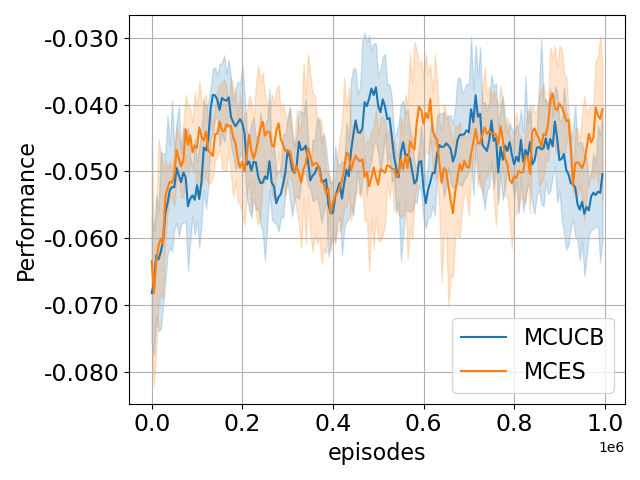}
        \caption{}
        \label{subfig:bj_score}
    \end{subfigure}
    \hfill
    \begin{subfigure}[b]{0.3\textwidth}
        \centering
        \includegraphics[width=\textwidth]{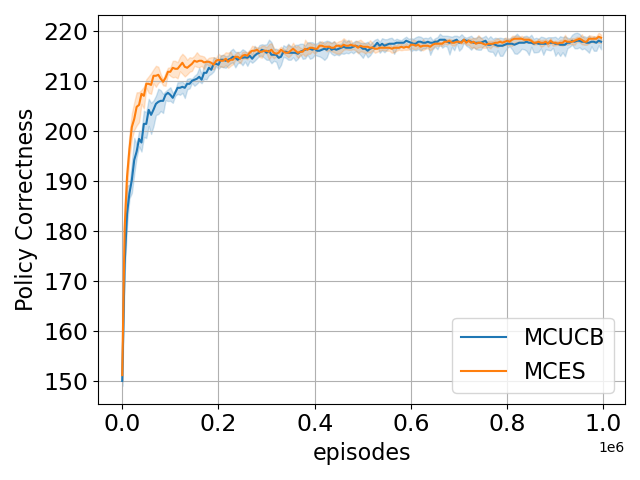}
        \caption{}
         \label{subfig:bj_policy}
    \end{subfigure}
    \hfill
    \begin{subfigure}[b]{0.3\textwidth}
        \centering
        \includegraphics[width=\textwidth]{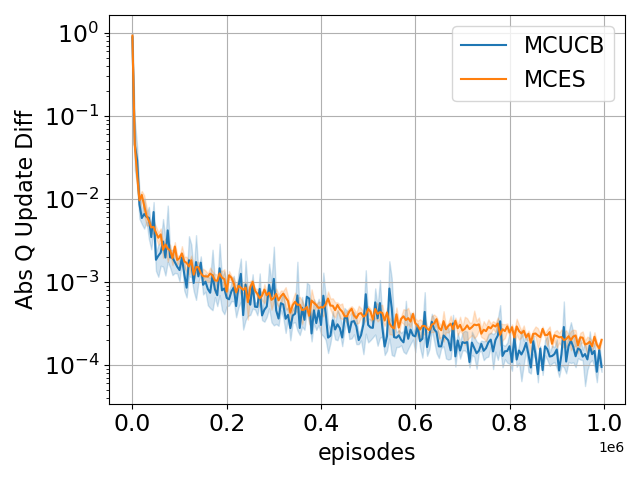}
        \caption{}
        \label{subfig:bj_q}
    \end{subfigure}
    \caption{Experiments running on the Blackjack. The x-axis is the total number of training episodes and the y-axis shows performance, policy convergence, and Q-value convergence respectively.}
    \label{fig:bj}
\end{figure}

% convergence of action-value functions is evaluated by summing the absolute differences between all estimated action-value functions and the true action-value functions after each training episode. Again, the true action-value functions are computed by the value iteration algorithm.

% Figure \ref{fig-main-blackjack} show that in Blackjack, compared to MCES, which uses a more ideal exploration scheme, MC-UCB can achieve a competitive performance and policy correctness (first and second figure). And in fact, MC-UCB is even faster in the convergence rate of action-value function. 

In Figure \ref{subfig:bj_score}, we see that MC-UCB obtains performance on par with MCES. Figure \ref{subfig:bj_policy} and \ref{subfig:bj_q} show that MC-UCB indeed converges and learns a good policy in Blackjack. The convergence is faster than MCES in terms of the action-value function, even if MCES has a more ideal exploration scheme. 

% In this case, the faster convergence rate of MC-UCB might be due to the fact that the initial state distribution of Blackjack has a good coverage of the state space, allowing the MC-UCB agent to also initialize from a large number of states, making the action-value function converges efficiently. 

% the Blackjack environment itself allows the learning agent to start with a large range of states so that this extra exploration and the UCB exploration make the estimated value functions for those state-action pairs near the terminal state converge efficiently.

\subsection{Cliff-Walking}
We now study the stochastic Cliff-Walking environment. The goal here is to successfully walk in a 2-D grid from a fixed starting point (bottom-left) to a terminal state (bottom-right) as soon as possible, without falling over the cliff. The state space contains each position in this 2-D grid, and the agent can move in 4 directions. The agent receives a reward of -0.01 for each step it takes, receives -1 for falling over the cliff in the bottom row, and receives 0 for reaching the goal. 
% the action space has four actions moving leftwards, upwards, rightwards, and downwards.
% We set a reward of -1 if the walker falls over the cliff at the bottom of the grid, and whenever the walker takes one step, it receives a reward of -0.01. Finally, reaching the terminal state in the lower right corner of the grid receives a reward of 0. 
Additionally, we add stochasticity to this environment: when moving to the right, a wind with probability $p$ can blow the agent upwards or downwards with equal probability. 
% If the agent moves to the right, there will be a wind occurring with a probability $p$ blowing it either upwards or downwards with the equal probability of $p/2$.

We compare MCES, MC-UCB with a fixed initialization (FI\_MC-UCB), and MC-UCB with a uniform initialization (UI\_MC-UCB). The metrics used here are: (1) average return over 5 testing episodes, (2) absolute difference between optimal and estimated Q-values; and (3) absolute difference between optimal and estimated V-values. In Figure \ref{fig:cliff}, each curve aggregates 3 grid sizes ($3\times2$, $6\times4$, $9\times6$), 5 wind probabilities ($p\in\{0, 0.3, 0.5, 0.7, 1\}$), and 5 seeds (thus 75 runs in total). We set the maximum episode length $M$ as the grid size and set the UCB constant $C=1$ throughout.

\begin{figure}[h]
    \centering
    \begin{subfigure}[b]{0.3\textwidth}
        \centering
        \includegraphics[width=\textwidth]{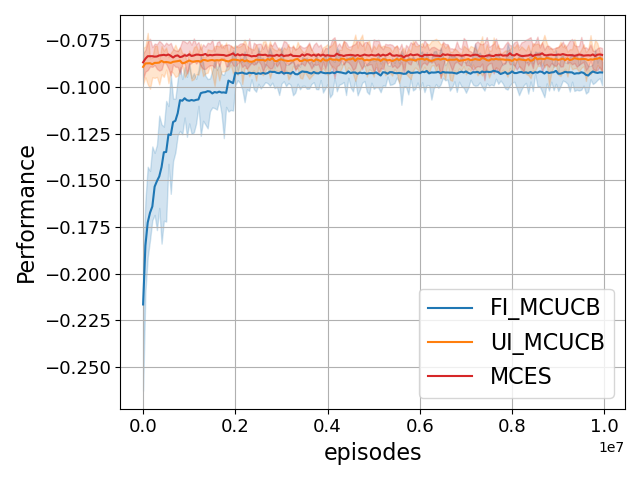}
        \caption{}
        \label{subfig:cliff_score}
    \end{subfigure}
    \hfill
    \begin{subfigure}[b]{0.3\textwidth}
        \centering
        \includegraphics[width=\textwidth]{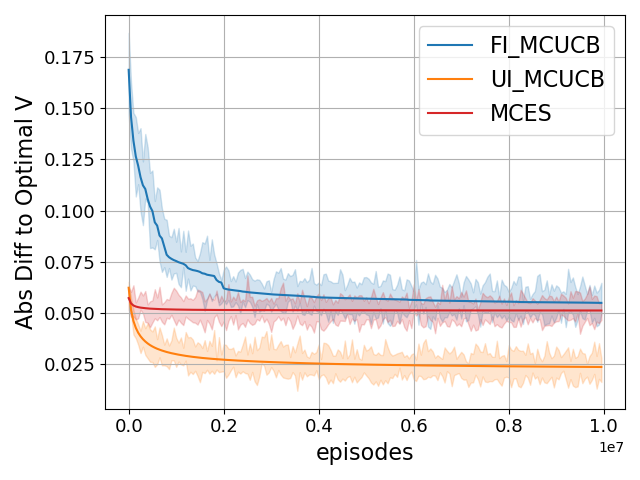}
        \caption{}
         \label{subfig:cliff_v}
    \end{subfigure}
    \hfill
    \begin{subfigure}[b]{0.3\textwidth}
        \centering
        \includegraphics[width=\textwidth]{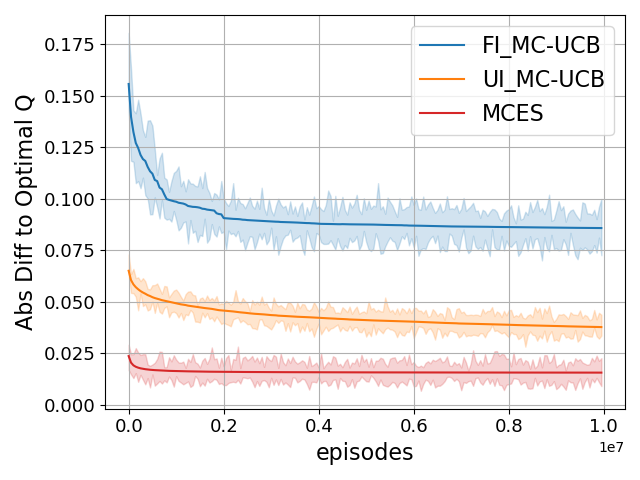}
        \caption{}
        \label{subfig:cliff_q}
    \end{subfigure}
    \caption{Experiments running on the Cliff-Walking. The x-axis is the total number of training episodes and the y-axis shows performance, V-value convergence, and Q-value convergence respectively.}
    \label{fig:cliff}
\end{figure}

Compared with MCES, FI\_MC-UCB can have competitive performance and V-value convergence as shown in Figure \ref{subfig:cliff_score} and \ref{subfig:cliff_v}. Additionally, allowing MC-UCB to explore starts can further improve the performance and surpass the V-value convergence of MCES. However, in terms of Q-value, Figure \ref{subfig:cliff_q} shows that MC-UCB variants have a lower convergence performance. We point out that slowly growing UCB exploration terms in the later learning phase may cause this when some Q values for suboptimal actions are updated infrequently, though infinitely often.

% We also note that in practice, MC-UCB can be further modified to improve the convergence rate. 
% More existing technical tricks or modifications to the MC-UCB framework can be added to make the convergence faster in practice, but we do not discuss them in this paper.

\section{Conclusion}

We conclude this paper by outlining some possible directions for future research. 
We have established the almost sure convergence for the MC-UCB algorithm for OPFF environments. Earlier in \citet{bertsekas1996neuro} and \citet{wang2022on}, it was shown that for some non-OPFF environments, the MCES algorithm does not converge. We conjecture that this is also the case for MC-UCB. Finding a counterexample remains an open question. As a first step for understanding the properties of MC-UCB with random episode lengths, it is important to understand under what conditions it does and does not converge. As a second step, it would be of interest to establish tight regret bounds for MC-UCB with random episode lengths with the OPFF assumption. It would also be interesting to extend the theory presented here to the MDP game context. For example, it would be of interest to study convergence in the context of 2-person 0-sum games for OPFF MDPs. This would allow us to move the theory in this paper a step closer toward a general convergence theory for AlphaZero \citep{silver2018general}.

%%%%%%%%%%%%%%%%%%%%%%%%%%%%%%%%%%%%%%%%%%%%%%%%%%%%%%%%%%%%%%%%
%% Appendices
%%%%%%%%%%%%%%%%%%%%%%%%%%%%%%%%%%%%%%%%%%%%%%%%%%%%%%%%%%%%%%%%
% \appendix

% \subsubsection*{Acknowledgments}
% \label{sec:ack}
% Use unnumbered third level headings for the acknowledgments. All acknowledgments, including those to funding agencies, go at the end of the paper. Only add this information once your submission is accepted and deanonymized. The acknowledgments do not count towards the 8--12 page limit.

%%%%%%%%%%%%%%%%%%%%%%%%%%%%%%%%%%%%%%%%%%%%%%%%%%%%%%%%%%%%%%%%
%% NOTE: THIS MARKS THE END OF THE "MAIN TEXT"
%%%%%%%%%%%%%%%%%%%%%%%%%%%%%%%%%%%%%%%%%%%%%%%%%%%%%%%%%%%%%%%%

%%%%%%%%%%%%%%%%%%%%%%%%%%%%%%%%%%%%%%%%%%%%%%%%%%%%%%%%%%%%%%%%
%% Bibliography
%%%%%%%%%%%%%%%%%%%%%%%%%%%%%%%%%%%%%%%%%%%%%%%%%%%%%%%%%%%%%%%%
\clearpage
\bibliography{rlj}
\bibliographystyle{rlj}

%%%%%%%%%%%%%%%%%%%%%%%%%%%%%%%%%%%%%%%%%%%%%%%%%%%%%%%%%%%%%%%%
% AUTHOR: If your paper has no supplementary materials, you may 
%         comment out the line below, which creates the title for
%         the supplementary materials.
%%%%%%%%%%%%%%%%%%%%%%%%%%%%%%%%%%%%%%%%%%%%%%%%%%%%%%%%%%%%%%%%
\beginSupplementaryMaterials
\appendix
\section{Almost Sure Convergence for Multi-Armed Bandits}
\label{appendix:mab}
Since the UCB algorithm was first introduced in the Multi-Armed Bandit (MAB) literature, we discuss some background related to our work in this section. The MAB problem has been widely studied in the past decades \cite{lattimore_szepesvari_2020}. In each round $t$ out of a total of $n$ plays, an agent pulls one arm $i\in\{1,2,\dots,K\}$ and receives the reward $X_{i,t}$ which is a random variable. We assume that successive plays of arm $i$ output rewards $X_{i,1}, X_{i,2}, \dots$ which are independent and identically distributed following an unknown law with an unknown bounded mean value $\mu_i$. Also, we assume rewards from different arms are independent but may not be identically distributed. 

One classic strategy to solve the exploration-exploitation dilemma in the MAB problem is called the UCB1 algorithm \cite{auer2002finite}. The agent maintains a one-sided upper confidence interval for the estimated mean reward $\hat{\mu}_i(n) = \frac{1}{T_i(n)}\sum_{t=1}^{T_i(n)}X_{i,t}$ for each arm $i$, where $T_i(n)$ is how many times the arm $i$ is pulled during the $n$ rounds of play. By the algorithmic design of UCB1, the choice of arm for the $t$-th round $I_t = \argmax_i \hat{\mu}_i(t-1) + \sqrt{\frac{C\log t}{T_i(t-1)}}$, where $C=2$ is a constant controlling the rate of exploration; and the second term will be $\infty$ if $T_i(t-1)=0$.

Much of the study of the MAB problem aims to develop efficient exploration-exploitation strategies in order to achieve low finite-horizon expected regret bound $R_n = \mu_* - \mathbb{E}[\frac{1}{n}\sum_{t=1}^{n}X_{I_t,t}]$ \cite{Bubeck2012}. However, instead of being satisfied with only the expected behavior of bandit algorithms, the convergence of pseudo-regret $\tilde{R}_n = \mu_* - \frac{1}{n}\sum_{t=1}^{n}X_{I_t,t}$ has also been studied. \citet{https://doi.org/10.48550/arxiv.1505.02865} shows that the pseudo-regret converges to 0 $w.p.1$ for two families of policies involving slow-growing functions. In particular, the authors mention that the convergence result of the g-ISM index policy proposed in section 3.2 of \citet{https://doi.org/10.48550/arxiv.1505.02865} can be extended to the UCB1 policy without providing a complete proof. We provide a direct, self-contained proof of almost sure convergence for the MAB problem using the UCB1 policy in Appendix \ref{App:proof of MAB}. 
\begin{lemma}[\textbf{Almost Sure Convergence for the MAB with UCB1}]
\label{single-bandit}
In the $K$-arm Bandit setting employing UCB1 exploration, let $X_{I_1,1}, X_{I_2,2}, \dots X_{I_n,n}$ be the rewards collected at each time of pulling one arm among $\{1,2,\dots, K\}$. Then,  $\overline{X}_n := \frac{1}{n}\sum_{t=1}^n X_{I_t,t} \rightarrow \mu_* $ w.p.1.
\end{lemma}
We emphasize that the MAB problem is not enough to model the two-node OPFF MDPs, since the latter one may contain an action leading to the self-cycle. Also, the analysis of the MAB problem is largely simplified, as the Q estimates for suboptimal actions can be guaranteed to converge to their optimal ones $w.p.1$ solely by the SLLN, which thus leads to the logarithmic bound for the first component of $T_i(n)$ in the line (\ref{log bound}) below.

\section{Proof of Lemma 1}
\label{App:proof of MAB}
\begin{proof}
We can decompose:
\begin{align*}
    \overline{X}_n &= \sum_{i \neq i^*} \frac{T_i(n)}{n}\cdot\frac{1}{T_i(n)}\sum_{u=1}^{T_i(n)} X_{i,u} +\frac{T_{i^*}(n)}{n}\cdot\frac{1}{T_{i^*}(n)}\sum_{u=1}^{T_{i^*}(n)} X_{i^*,u}
\end{align*}
where $i^*$ is the optimal arm that yields the reward with the highest mean reward $\mu_*$.

Let $\Omega$ be the underlying sample space. By the design of the UCB1 algorithm, each arm will be pulled infinitely often if we conduct infinite rounds of play. In other words, $T_i(n)\rightarrow\infty$ as $n\rightarrow\infty$ for any arm $i$. %In other words, let $\Lambda_i = \{\omega\in\Omega: T_i(n)(\omega)\overset{n\rightarrow\infty}{\rightarrow}\infty\}$ and then we have $P(\Lambda_i) = 1$ for each arm $i$. 
By the Strong Law of Large Number for $i.i.d$ random variables $\{X_{i,u}\}_{u=1}^{T_i(n)}$, the estimated mean $\hat{\mu}_i(n) = \frac{1}{T_i(n)}\sum_{u=1}^{T_i(n)} X_{i,u}$ converges to the true mean $\mu_i$ $w.p.1$, for any arm $i$. That is, the set $\Lambda_i = \{\nu\in\Omega: \hat{\mu}_{i}(n)(\nu)\overset{n\rightarrow \infty}{\longrightarrow} \mu_i\}$ has probability measure 1. Moreover, let $\Lambda = \cap_i^K \Lambda_i$ and we have $P(\Lambda) = 1$. Then, what remains to show is $\frac{T_i(n)}{n} \rightarrow 0\ w.p.1$, for any $i\neq i^*$.

For any $i \neq i^*$, according to the inequality 8.4 of Section 8.1 in \cite{lattimore_szepesvari_2020}, given any $\epsilon_1$ and a positive and increasing function $f$, we can bound:
\begin{align}
    T_i(n) &= \sum_{t=1}^n \mathbb{I} \{I_t=i\}\\ \label{log bound}
    &\leq  \sum_{s=1}^n \mathbb{I}\{\hat{\mu}_{is} + \sqrt{\frac{2\log f(n)}{s}} \geq \mu_* - \epsilon_1\}\\
    & + \sum_{t=1}^n \mathbb{I}\{\hat{\mu}_{i^*}(t-1) + \sqrt{\frac{2\log f(t)}{T_{i^*}(t-1)}} \leq \mu_* - \epsilon_1\}
\end{align}
where $\hat{\mu}_{is}$ is defined to be the estimated mean reward of arm $i$ by using the first $s$ samples; if $s>T_i(n)$, we will just set $\hat{\mu}_{is} = \hat{\mu}_i(n)$.

Let $\Delta_i = \mu_* - \mu_i$ for any $i$. On the one hand, for any $\omega \in \Lambda$, given $\epsilon_2 > 0$ satisfying $\Delta_i > \epsilon_1 + \epsilon_2$, there exists $N_1 > 0$ such that for any $s > N_1$, $|\hat{\mu}_{is}(\omega)-\mu_i| < \epsilon_2$. Thus, we can further bound the first term:
\begin{align*}
    &\sum_{s=1}^n \mathbb{I}\{\hat{\mu}_{is}(\omega) + \sqrt{\frac{2\log f(n)}{s}} \geq \mu_* - \epsilon_1\} \\
    & = \sum_{s=1}^n \mathbb{I}\{\hat{\mu}_{is}(\omega) - \mu_i + \sqrt{\frac{2\log f(n)}{s}} \geq \Delta_i - \epsilon_1\} \\
    & \leq \sum_{s=1}^n \mathbb{I}\{|\hat{\mu}_{is}(\omega) - \mu_i| + \sqrt{\frac{2\log f(n)}{s}} \geq \Delta_i - \epsilon_1\} \\
    & \leq N_1 + \sum_{s=N_1+1}^n \mathbb{I}\{\epsilon_2 + \sqrt{\frac{2\log f(n)}{s}} \geq \Delta_i - \epsilon_1\} \\
    & = N_1 + \sum_{s=N_1+1}^n \mathbb{I}\{s\leq \frac{2\log f(n)}{(\Delta_i -\epsilon_1 - \epsilon_2)^2}\}\\
    & = \floor {\frac{2\log f(n)}{(\Delta_i -\epsilon_1 - \epsilon_2)^2}}
\end{align*}
where the last equation is guaranteed if the function $f$ is chosen such that $\frac{\log f(n)}{n}\rightarrow 0$ as $n\rightarrow \infty$.

On the other hand, for any $\omega\in\Lambda$, given $\epsilon_3 > 0$ satisfying $\epsilon_1 \geq \epsilon_3$, there exists $N_2 > 0$ such that for any $t > N_2$, $|\hat{\mu}_{i^*}(t-1)(\omega)-\mu_{i^*}| < \epsilon_3$, we can also bound the second term:
\begin{align*}
    &\sum_{t=1}^n \mathbb{I}\{\hat{\mu}_{i^*}(t-1)(\omega) + \sqrt{\frac{2\log f(t)}{T_{i^*}(t-1)}} \leq \mu_* - \epsilon_1\}\\
    & \leq \sum_{t=1}^n \mathbb{I}\{\hat{\mu}_{i^*}(t-1)(\omega) \leq \mu_* - \epsilon_1\}\\
    & \leq N_2 + \sum_{t=N_2+1}^n \mathbb{I}\{\epsilon_1 < \epsilon_3\}\\
    & = N_2
\end{align*}

Now, from the above three inequalities, it is easy to see: for any $\omega\in\Lambda$, $\frac{T_i(n)(\omega)}{n} \rightarrow 0$, as $n\rightarrow \infty$, for any $i\neq i^*$. Since $P(\Lambda)=1$, $\frac{T_i(n)}{n} \rightarrow 0$ $w.p.1$, for any $i\neq i^*$. And therefore, we finally have: $\overline{X}_n \rightarrow \mu_* $ w.p.1..
\end{proof}

\end{document}